\documentclass[conference]{IEEEtran}
\IEEEoverridecommandlockouts
\usepackage{dsfont}
\usepackage{algpseudocode}
\usepackage{tablefootnote}
\usepackage{amsthm, xpatch, bm} %bm: a pơerful package for bolding the symbol
\usepackage[thinc]{esdiff}
\usepackage[bookmarks=false,hidelinks]{hyperref}
\hypersetup{
     colorlinks = true,
     linkcolor = blue,
     anchorcolor = black,
     citecolor = blue,
     filecolor = blue,
     urlcolor = black,
}
\usepackage{subcaption}
\usepackage[font=small,skip=0pt]{subcaption}
\usepackage{url}
\usepackage{cite}
\usepackage{xcolor}
\usepackage[nolist]{acronym}
\usepackage{graphicx, graphics} 
\usepackage{float}
\usepackage{threeparttable}
\usepackage{cancel}
\usepackage[official]{eurosym}
\usepackage{fancyhdr}
\usepackage{comment}
\usepackage{chemformula}
\usepackage{setspace}
\usepackage[T1]{fontenc} % optional
\usepackage{amsmath, mathtools, tcolorbox}
\usepackage{mathptmx}
\usepackage{cleveref}
\usepackage[cmintegrals]{newtxmath}
\usepackage{bm} % optional
\usepackage[colorinlistoftodos,prependcaption,textsize=tiny]{todonotes}
\usepackage{amsmath,amsfonts}
\usepackage{empheq}
\usepackage{array}
\usepackage{textcomp}
\usepackage{stfloats}
\usepackage{url}
\usepackage{verbatim}
\def\BibTeX{{\rm B\kern-.05em{\sc i\kern-.025em b}\kern-.08em
    T\kern-.1667em\lower.7ex\hbox{E}\kern-.125emX}}
\usepackage{balance}
\usepackage{diagbox}
\usepackage{algorithm}
\usepackage{pdfrender}
\usepackage{xcolor}
\usepackage{glossaries}
\usepackage{makecell}
\usepackage[usestackEOL]{stackengine}
\usepackage{array, multirow, cellspace}
\usepackage{pifont}% http://ctan.org/pkg/pifont
\usepackage{lipsum, tabularx, ragged2e, multirow}
\usepackage{hyperref}
\usepackage[utf8]{inputenc}
\usepackage{enumitem}
\usepackage{wrapfig}
\usepackage{colortbl,hhline}
\usepackage{flushend}
\usepackage{balance}
\usepackage{booktabs}

\usepackage[switch, columnwise]{lineno}
\usepackage{blindtext}

\def\({\left(}
\def\){\right)}

\def\[{\left[}
\def\]{\right]}

\usepackage{glossaries}

\usepackage{soul}

\usepackage[top=0.765in, bottom=1.05in, left=0.625in, right=0.625in]{geometry}
\def\BibTeX{{\rm B\kern-.05em{\sc i\kern-.025em b}\kern-.08em
    T\kern-.1667em\lower.7ex\hbox{E}\kern-.125emX}}
    
\def\delequal{\mathrel{\ensurestackMath{\stackon[1pt]{=}{\scriptscriptstyle \Delta}}}}

\newcolumntype{P}[1]{>{\centering\arraybackslash}p{#1}}

\makeatletter
\xpatchcmd{\proof}{\topsep6\p@\@plus6\p@\relax}{}{}{}
\makeatother

\newtheorem{theorem}{Theorem}

\newtheorem{definition}{Definition}

\newcommand{\proposed}{{\small \textsf{CoCoGen}}}
\newcommand{\dgen}{d^{{\text{gen}}}_n} %d^gen_n
\newcommand{\dgenopt}{d^{{*,\text{gen}}}_n} %d*^gen_n
\newcommand{\dloc}{d^{{\text{loc}}}_n} %d^loc_n
\newcommand{\esilontotaldgen} {\epsilon({\boldsymbol{d}^\text{\text{gen}}})} %E(d^gen)

\newcommand{\esilondngen}{\epsilon_{n}(d^\text{\text{gen}}_{n}, \boldsymbol{d}^\text{\text{gen}}_{-n})} %E(d^gen_n, d^gen_{-n})

\newcommand{\esilondgenother}{\epsilon_{n'}(d^\text{\text{gen}}_{n'}, \boldsymbol{d}^\text{\text{gen}}_{-n'})} %E(d^gen_n', d^gen_{-n'})

\newcommand{\dgenmin}{d^{{\text{gen}}}_{\text{min}}}
\newcommand{\dgenmax}{d^{{\text{gen}}}_{\text{max}}}
\newcommand{\potfunc}{F(\boldsymbol{d}^{\text{gen}})}

\def\BibTeX{{\rm B\kern-.05em{\sc i\kern-.025em b}\kern-.08em
    T\kern-.1667em\lower.7ex\hbox{E}\kern-.125emX}}
\begin{document}

\title{A Coopetitive-Compatible Data Generation Framework for Cross-silo Federated Learning}
% \title{CoCoGen: Coopetitive-Compatible Data Generation Framework for Cross-silo Federated Learning}
% In the long version, we can change the title to CoCoGen: Coopetitive-Compatible Data Generation Framework for Cross-silo Federated Learning, and our framework will be named CoCoGen instead of C2DGF. 
% You can also now change the method's name to CoCoGen

\author{\IEEEauthorblockN{Thanh Linh Nguyen, Quoc-Viet Pham}
\IEEEauthorblockA{{School of Computer Science and Statistics},
{Trinity College Dublin, The University of Dublin,}
Dublin, Ireland \\
Email: \{tnguyen3, viet.pham\}@tcd.ie
}}

% \author{\IEEEauthorblockN{Anonymous Authors}
% }

\maketitle

\begin{abstract}
Cross-silo federated learning (CFL) enables organizations (e.g., hospitals or banks) to collaboratively train artificial intelligence (AI) models while preserving data privacy by keeping data local. While prior work has primarily addressed statistical heterogeneity across organizations, a critical challenge arises from economic competition, where organizations may act as market rivals, making them hesitant to participate in joint training due to potential utility loss (i.e., reduced net benefit). Furthermore, the combined effects of statistical heterogeneity and inter-organizational competition on organizational behavior and system-wide social welfare remain underexplored. In this paper, we propose {\proposed}, a coopetitive-compatible data generation framework, leveraging generative AI (GenAI) and potential game theory to model, analyze, and optimize collaborative learning under heterogeneous and competitive settings. Specifically, {\proposed} characterizes competition and statistical heterogeneity through learning performance and utility-based formulations and models each training round as a weighted potential game. We then derive GenAI-based data generation strategies that maximize social welfare. Experimental results on the Fashion-MNIST dataset reveal how varying heterogeneity and competition levels affect organizational behavior and demonstrate that {\proposed} consistently outperforms baseline methods.
% in enhancing social welfare.
\end{abstract}


%%% Coloring the comment as blue
\newcommand\mycommfont[1]{\footnotesize\ttfamily\textcolor{blue}{#1}}
% \SetCommentSty{mycommfont}

% \SetKwInput{KwInput}{Input}                % Set the Input
% \SetKwInput{KwOutput}{Output}              % set the Output
% \renewcommand{\algorithmicforall}{\textbf{for each}}

\section{Introduction}

The development of artificial intelligence (AI)-driven services and products increasingly relies on access to large-scale, high-quality data \cite{alabdulmohsin2022revisiting}. In various domains, such as healthcare and finance, data is geographically dispersed across different organizations (e.g., hospitals or banks) and data sharing is constrained by privacy regulations (e.g., GDPR \cite{regulation2016regulation}). Such restrictions pose significant barriers for organizations of all sizes from acquiring sufficient, high-quality data, especially for training generative AI (GenAI) models \cite{gnyawali2009co, openllm}. Cross-silo federated learning (CFL) \cite{kairouz2021advances} offers a solution by enabling organizations to jointly train a global model while keeping raw data local, preserving data privacy and regulatory compliance.
% . Only trained model updates are exchanged with a central server, preserving data privacy and regulatory compliance.

Although CFL provides a privacy-preserving framework for collaborative AI model training, organizations remain cautious about contributing their data and computational resources. A key reason is the presence of \textit{coopetition}, where \textit{cooperative} organizations may also be market \textit{competitors} on downstream tasks, making them self-interested and concerned about revealing competitive advantages through cooperation \cite{huang2023duopoly}. Beyond competition, \textit{statistical heterogeneity} across organizations further complicates collaboration, often causing the global model to diverge from optima  \cite{kairouz2021advances}. Thus, organizations may receive a poorly generalized global model, 
% which functions as a public good \cite{tang2021incentive}, 
leading to low individual utility \cite{tang2021incentive}. This outcome ultimately reduces the \textit{social welfare}. 

Prior work investigates how competition influences learning performance, utility (i.e., difference between gains and losses) outcomes, and organizational decision-making in CFL \cite{NEURIPS2024_62ffefbe, tsoy2024strategic, huang2023duopoly}. \textit{However, statistical heterogeneity and its influence have been underexamined under the competitive nature among organizations in CFL, limiting the applicability of existing mechanisms in real-world applications.} 
To address the interplay between competition and heterogeneity, Huang et al. in \cite{huang2024coopetition} analyzed coopetition via a two-period game-theoretic model, where organizations make joint decisions on collaboration and pricing. Notably, their results show that FL cooperation is sustainable only when learning performance improvements from collaboration outweigh the competitive losses. \textit{However, their work lacks an explicit theoretical characterization of statistical heterogeneity and competitive intensity, nor quantify their effects on learning performance, social welfare, and organizational decision-making, limiting generalizability and scalability.} Recently, GenAI has been shown to be an effective data augmentation solution to mitigate statistical heterogeneity in CFL~\cite{10398474}. GenAI provides a scalable method to close data distribution gaps in competitive CFL scenarios, where competition limits data sharing and increases distributional mismatch across silos. Given that data generation is resource-intensive, organizations need to optimize the use of GenAI-based augmented data and computational resources to enhance and sustain CFL cooperation (e.g., maximizing the global FL model performance and social welfare \cite{tang2021incentive}). 

Motivated by these challenges and discussions, we aim to address the following research questions (RQs) in this work:
\noindent \textbf{\#RQ1:} \textit{How does the coexistence of statistical heterogeneity and competition affect organizations' decision-making and social welfare?}

\noindent \textbf{\#RQ2:} \textit{In the presence of statistical heterogeneity and competition, how can we design an effective GenAI-based data generation mechanism for organizations?}

To answer these RQs, we propose a \underline{Co}opetitive-\underline{Co}mpatible Data \underline{Gen}eration Framework for CFL, dubbed {\proposed}. The term \textit{compatible} reflects the framework's design goal of aligning competition with cooperative incentives, enabling competitive organizations to benefit from cooperation. {\proposed} jointly models inter-organizational competition and statistical heterogeneity, and their impacts in the GenAI era, leveraging the coopetition concept from economics \cite{gnyawali2009co} and potential games \cite{la2016potential}. {\proposed} integrates theoretical modelling and extensive experiments to provide practical insights for CFL coopetition under heterogeneous settings.

The main contributions of this paper are as follows: 
\begin{itemize}   
\renewcommand{\labelitemi}{$\rhd$}

    \item \textit{CFL coopetition framework in heterogeneous settings:} We propose a coopetitive-compatible data generation framework, called {\proposed}, which models statistical heterogeneity and inter-organizational competition through learning performance and organization utility functions. {\proposed} analyzes how these factors impact organizations' behavior and system-wide social welfare in CFL.
    
    \item \textit{Potential game-based data generation mechanism:} We prove that each training round is a weighted potential game, capturing the strategic interactions among rational, competitive, and heterogeneous organizations. We derive Nash equilibrium solutions using Karush–Kuhn–Tucker (KKT) conditions and the fixed-point iteration method, enabling each organization to determine its data generation strategy while maximizing collective social welfare.
    
    \item \textit{Performance evaluation and Insights:} We conduct extensive experiments on the Fashion-MNIST dataset under individual rationality and budget balance constraints. Results show that increasing heterogeneity and competition levels reduce social welfare and require a larger volume of GenAI-generated data from organizations. Notably, {\proposed} consistently outperforms baseline methods in improving social welfare across diverse CFL settings.
\end{itemize}
\section{Proposed Framework and problem formulation}
\label{sec:system_model}
%%%%%%%%%%%%
\subsection{Overview}

% Original citations for this claim: \cite{gnyawali2009co, dorn2016levels, thomason2013several}
We leverage the concept of coopetition from economics \cite{gnyawali2009co} to define a new concept for CFL.
\begin{definition} [CFL Coopetition]
\label{def:coopetition}
\underline{\textbf{CFL coopetition}} characterizes the dual strategy of \textbf{cooperative} engagement in joint model training to develop a shared global model, \textbf{competitive} utilization of this model to maximize their revenue outcomes or market performance, reflecting their inherent market rivalry.
\end{definition}

As illustrated in Fig.~\ref{fig:system_model}, we consider a practical CFL architecture with multiple coopetitive organizations and a trustworthy central server, which can be either rented or jointly established for coordination. These organizations, whose utilities depend on the model performance improvements across all participants \cite{wu2022mars}, form a consortium for model training.

We denote by $\mathcal{N} \delequal \{1,\dots,N\}$ the set of $N$ organizations, where each organization $n \in \mathcal{N}$ has a processing capability $f_n$, and a dataset consisting of original local data or/and new data generated by GenAI models. Organizations collaboratively participate in training a naive global model over a set of training rounds $\mathcal{T} \delequal \{1,..., t,..., T\}$ to improve performance on downstream tasks used for their services or products. Let $\mathcal{D}^{\text{loc}}_n$ and $\mathcal{D}^{\text{gen}}_n$ denote the original local data and GenAI-generated data of organization $n$, respectively, with corresponding sizes $\dloc \delequal |\mathcal{D}^{{\text{loc}}}_n|$ and $\dgen \delequal |\mathcal{D}^{{\text{gen}}}_n|$. In traditional CFL settings, organization $n$ trains its local model solely on dataset $\mathcal{D}^{{\text{loc}}}_n$. Differentially, in our work, each organization~$n$ adopts a mixed local dataset of $\mathcal{D}^{{\text{mix}}}_n = \mathcal{D}^{{\text{loc}}}_n \cup \mathcal{D}^{{\text{gen}}}_n$ and $d_n^\text{\text{mix}} \delequal |\mathcal{D}^{{\text{mix}}}_n| = d_n^\text{\text{loc}} + d_n^{\text{gen}}$ for its local model update, where we apply the equal-generated data allocation strategy to each class. Let $\boldsymbol{d}^{\text{gen}} \delequal \{\dgen: \forall{n} \in \mathcal{N}\}$ denote the GenAI-augmented data size strategy profile of all organizations in each round $t$, and let $\boldsymbol{d}^{\text{gen}}_{-n}$ be the data size strategy profile of all organizations excluding organization $n$ (i.e., $\boldsymbol{d}^{\text{gen}}_{-n} \delequal \{{d}^{\text{gen}}_{n'},{\forall n' \neq n, n' \in \mathcal{N}}\})$.

% We consider a $C$ class image classification problem for our FL task. The data synthesis strategy of organization $n$ can be represented as $\{d^{\text{gen}}_{n, c}\},{\forall c \in C}$. The strategy should meet the constraint $\sum_c d^{\text{gen}}_{n, c} = {D}^{{\text{gen}}}_n$, where $d^{\text{gen}}_{n, c}$ represents the amount of synthetic data for the $c$-th class on the organization $k$\footnote{In this work, we apply the equal generated data allocation strategy to each class, that has been empirically proved to be efficient compared to other strategies such as water-filling-based or inverse \cite{ye2023federated}.}. It is noteworthy that if organizations do not have enough computing resources or they already have diverse datasets, they will not use AIGC models to generate more data to enhance local datasets (e.g., $d^{\text{gen}}_{n, c} = 0$ for specific classes or ${D}^{{\text{gen}}}_n = 0$). Let $\boldsymbol{\pi} = (\pi_n, n \in \mathcal{N})$ denote the contribution strategy vector of $N$ organizations, where $\pi_n = \{D^{\text{mix}}_n, f_n\}$.

\begin{figure}[t!]
	\centering
	\includegraphics[width=0.81\linewidth]{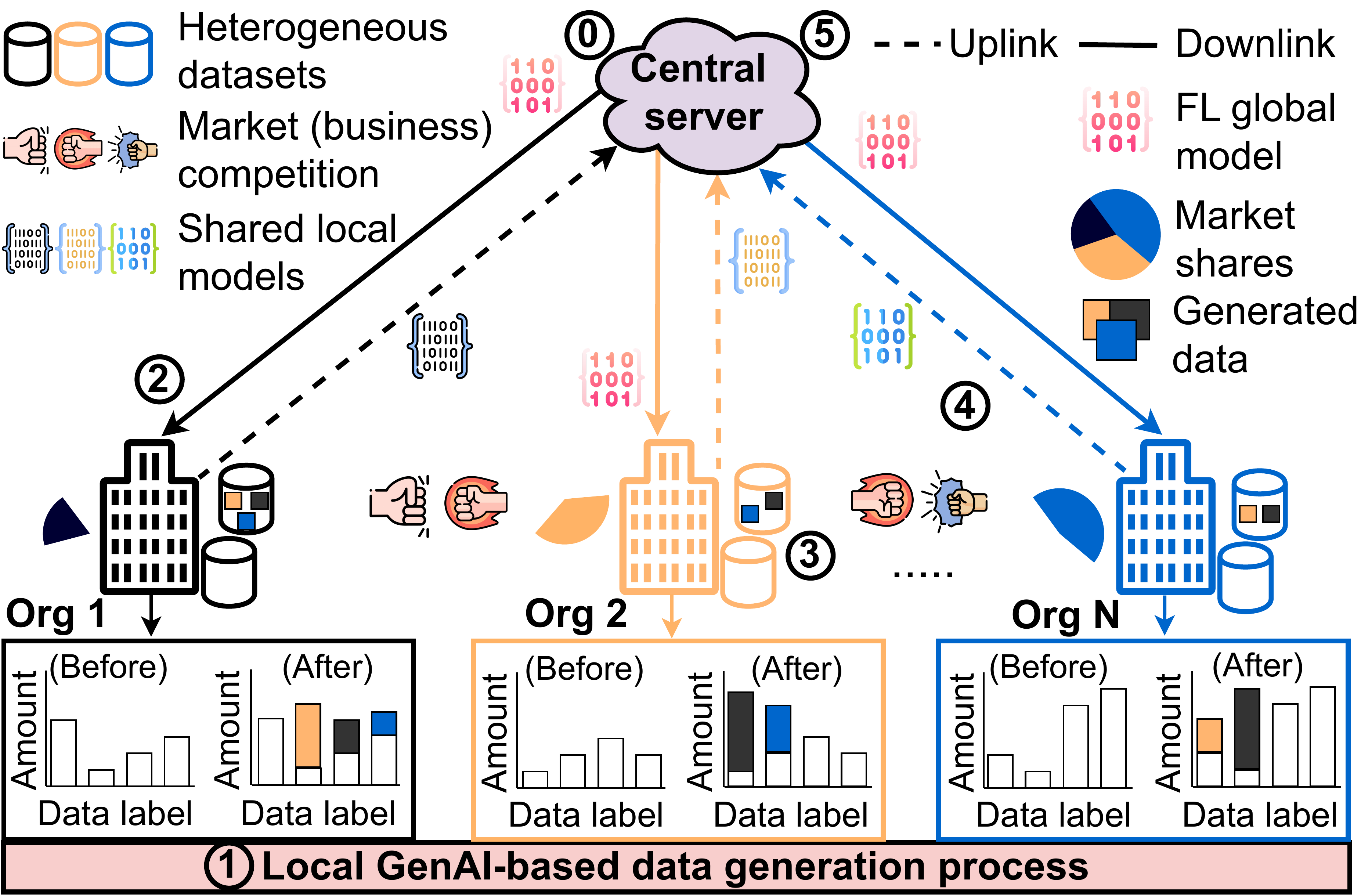}
	\caption{{\proposed} architecture with workflow.}
    \vspace{-14pt}
    \label{fig:system_model}
\end{figure} 

\subsection{GenAI-based Cross-silo Federated Learning Model}
Organizations jointly train a global model, each using a local neural network matching the global architecture.

The one-round training process of CFL comprises: \textbf{0) Global model initialization \& Organization strategy profile allocation}. The central server initializes a global model $\boldsymbol{w}^{0}$. We denote $\boldsymbol{w}^{t}$ as parameters of the global model in the training round $t$. Besides, based on gathered information (e.g., high-level data distribution, computing capabilities), which is approved among organizations for learning tasks, the central server can calculate and send data generation strategies (elaborated in Sec.~\ref{sec:proposed_mechanism}) to organizations. \textbf{1) Local GenAI-based data generation.} After receiving guidance, organizations put their efforts into augmenting their local datasets (if necessary).  \textbf{2) Global model downloading}. In the training round $t$, each organization $n$ downloads the global model parameter $\boldsymbol{w}^{t-1}$. \textbf{3) Local model updating}. Each organization $n$ selects local data size ${d}^{\text{mix}}_n$ and uses $f_n$ computational resources to train $\boldsymbol{w}^{t-1}$ over one local update performing batch gradient descent to obtain a newly-trained local model update. The local loss function of organization $n$ in the training round $t$ 
% as $L^{t}_n(\boldsymbol{w}^{t-1})$, which 
is as follows:
\begin{align}
\label{eq:local_loss}
L^{t}_n(\boldsymbol{w}^{t-1}) = \frac{1}{d_{n}^{{\text{mix}}}} \sum\nolimits_{i=1}^{d_{n}^{{\text{mix}}}} l(\boldsymbol{w}^{t-1}; \boldsymbol{x}_i),
\end{align}
where $\boldsymbol{x}_i$ is a data sample in the mixed dataset $\mathcal{D}^{{\text{mix}}}_n$ and $l(\boldsymbol{w}^{t-1}; \boldsymbol{x}_i)$ is the loss function under training $\boldsymbol{w}^{t-1}$ over $\boldsymbol{x}_i$. Based on \cite{chen2018my, wang2020machine}, we define the inverse power-law function of the lower bound of $n$'s local model performance as follows: 
\begin{align}
\label{eq:local_loss_error}
\epsilon_{n} = \alpha(d^{{\text{loc}}}_{n} + d^{{\text{gen}}}_{n})^{-\beta} - \delta,
\end{align}
where $\alpha > 0,\text{ } \beta > 0, \text{ and } \delta \ge 0$ are experimentally determined hyperparameters\footnote{We can fit the performance curves using two ways such as extrapolation or approximation depending on the learning tasks' complexity \cite{wang2020machine}.} (conducted in Sec.~\ref{sec:performance_evaluation}). This function illustrates that model performance depends on both local and GenAI-augmented data, underscoring the critical importance of optimizing the volume of generated data. \textbf{4) Local model uploading}. The organization $n$ uploads its trained model updates $\boldsymbol{w}_n^{t}$ to the server. \textbf{5) Model aggregation}. Finally, the server aggregates local model updates using an aggregation algorithm (e.g., FedAvg \cite{mcmahan2017communication}) to 
get a new global model $\boldsymbol{w}^{t}$. 
% (e.g., FedAvg \cite{mcmahan2017communication}). 
% The global loss function is as follows:
% \begin{align}
% \label{eq:global_loss}
% L^t(\boldsymbol{w}^{t-1}) = \sum_{n \in \mathcal{N}} \frac{d_{n}^{{\text{mix}}}}{\sum^{N}_{i=1} d_{i}^{{\text{mix}}}} L_{n}^{t}(\boldsymbol{w}^{t-1}_n).
% \end{align}

%%%%%%%%%%%%
\subsection{Data Generating, Training, and Aggregating Cost}
% The duration required for each training round including time for downloading the global model, generating AIGC-augmented data, training global model, and uploading local model update of each organization is as follows
% \begin{align}
% \label{eq:oneraining_round}
% \tau({\boldsymbol{\pi}}) &  = \max_{n \in \mathcal{N}}\{{T_n^{\text{\tiny DL}} + T_n^{\text{\tiny UP}} + T_n^{\text{\tiny UL}}}\} \\
% & = \max_{n \in \mathcal{N}}\{{T_n^{\text{\tiny DL}} + \frac{\mu_n D_n^{\text{gen}}}{f_n} + \frac{\eta_n K D_n^{\text{mix}}}{f_n}  + T_n^{\text{\tiny UL}}}\},
% \end{align}
Let $\mu_n$ and $\eta_n$ be the computation workload required for the organization $n$ to process and generate a single training data sample, respectively.
% Given the limited resources of organizations and to ensure training efficacy for each learning-based task, there is a fixed total training time $T \in \mathbb{R}_{+}$. Hence, the total number of training rounds for each organization $n$ is calculated by 
% \begin{align}
% \label{eq:no_of_round}
%  \upsilon({\boldsymbol{\pi}}) = T / \tau({\boldsymbol{\pi}})
% \end{align}
The energy consumption in one training round of organization $n$ can be profiled by \cite{pandey2020crowdsourcing}
\begin{align}
\label{eq:computation_overhead}
E_{n}^{\text{cmp}} = \kappa_n (\eta_n d_{n}^{{\text{mix}}} + \mu_n d_{n}^{{\text{gen}}}) f_n^2,
\end{align}
where $\kappa_n$ is the effective capacitance coefficient of organization $n$'s computational chipset. The total cost for computation tasks of $n$ in each training round can be represented by
% \footnote{Because all organizations share the same model architecture, the size of the uploaded trained model updates for each organization is fixed. Thus, the communication cost is neglected in this work.}
\begin{align}
\label{eq:cost}
C_{n} = C_n^{\text{cmp}} E_{n}^{\text{cmp}},
\end{align}
where $C_n^{\text{cmp}}$ represents the computational operating cost per energy unit. Besides, each organization pays the central server the same fee $C_{0}$ for calculating the data generation strategy profile and aggregating the global model. It is noted that the size of model updates for each organization is fixed, and the communication cost is neglected in this work.
% Therefore, the number of communication rounds can be calculated by dividing the pre-defined total training time by the duration of each training round, which is expressed as follows \footnote{Each training round is considered complete when the server receives all local model updates required for the round from organizations. Due to the system and statistical heterogeneity, each client will have a different complete time.}
% \begin{align}
% \label{eq:comm_rounds}
% \tau(\boldsymbol{\pi}) = \frac{T}{\max_{n, n' \in \mathcal{N}, n \neq n'} [T_n({\pi_n}), T_{n'}({\pi_n'})]}.
% \end{align}

% To ensure the training efficiency, we assume that all the above phases will be completed within a pre-defined time $T \in R_*^+$, i.e., $T_n(D_n^{\text{mix}}, f_n) \le T$.

%%%%%%%%%%%%
\subsection{Organization Utility given Coopetition and Heterogeneity}
%%%%%
% CFL organizations exhibit \textbf{cooperation} and \textbf{competition} reflecting their rational, self-interested, and competitive nature. 
This subsection characterizes these relationships by quantifying the gains from cooperation and the shared global model, and the competitive losses arising from strategic interactions.

\textit{1) \textbf{Gain obtained from coopetition.}} Organization $n$ contributes a dataset of size ${d}^{\text{mix}}_{n}$ to achieve a high-accuracy global model (i.e., minimizing training loss). Following \cite{tang2021incentive}, we quantify the contribution of each organization as a function of the trained global model performance (i.e., training model error). We denote $\epsilon$ as the global model performance after one training round. The smaller the value of the trained global performance, the better the fit of the global model to the training data. According to \cite{pandey2020crowdsourcing}, we establish the relationship between the global model performance $\epsilon$ under all organizations' data contribution strategy profile ${\boldsymbol{d}^{\text{gen}}}$ and the local model performance $\epsilon_n$, which is profiled as follows:
\begin{align}
\label{eq:data_size_loss}
\esilontotaldgen = \text{exp}\left(\frac{\frac{1}{N}\sum\nolimits_{n \in \mathcal{N}}{\epsilon_{n}} - 1}{\varrho}\right),
\end{align}
where $\varrho > 0$ is a fine-tuning parameter based on learning tasks.

The performance of organization $n$'s locally trained model, representing its contribution, can be expressed as $\epsilon_{n}(d^{\text{gen}}_{n}, \boldsymbol{d}^{\text{gen}}_{-n}) = \esilontotaldgen - \epsilon(\boldsymbol{d}^{\text{gen}}_{-n})$.
The revenue of the organization $n$, obtained from its contribution $\dgen$ to improve the global performance, is given by 
\begin{align}
\label{eq:gain}
r_{n}(d^{\text{gen}}_{n}, \boldsymbol{d}^{\text{gen}}_{-n}) = \psi_{n} \left[\epsilon_0 - \esilontotaldgen\right],
\end{align} 
where $\psi_{n}  \ge 0$ is the organization $n$'s revenue per global model's performance unit or its valuation on global model precision \cite{tang2021incentive, huang2023duopoly} and $\epsilon_0$ is the performance of the global model before the training process. Intuitively, a higher $\psi_n$ means higher profitability that the organization $n$ gains from the global model-related services they can create. Then, we define the competitive intensity in CFL coopetition.
\begin{definition} [Competitive Intensity]
\label{def:competition_intensity}
CFL competitive intensity $\gamma_{n, {n'}} \in [0,1]$, $n \neq n'$, between organizations $n$ and $n'$ represents their market rivalry. This metric captures the degree of strategic overlap of similarity in global model-derived products and congruence in their target customer segments.
\end{definition}
To guarantee fairness within the consortium, organizations contributing less effort (e.g., data size or computing resources)  in improving the global model performance are required to redistribute their payoff to that of higher efforts \cite{yuan2023tradefl}.
\begin{definition} [Payoff Redistribution]
\label{def:pay_off_redistribution}
Organization $n$ receives payoff redistribution from competing organizations $n'$  proportional to its marginal contribution to global model performance improvement. This compensation can be given by
\begin{align}
\label{eq:pay_off_redistribution_1}
p_{n,n'} =  \xi \gamma_{n, {n'}}\left[\esilondngen - \esilondgenother\right],
\end{align}
where $\xi \ge 0$ is the payoff compensation for each contribution gap unit. Therefore, the total payoff of $n$ gained from its competitors $n'$, can be calculated by $P_{n}(d^{\text{gen}}_{n}, {\boldsymbol{d}^{\text{gen}}_{-n}}) = \sum_{n' \in \mathcal{N}} {p_{n, n'}}.$
\end{definition}
% \begin{align}
% \label{eq:pay_off_redistribution_2}
% P_{n}(d^{\text{gen}}_{n}, {\boldsymbol{d}^{\text{gen}}_{-n}}) = \sum_{n' \in \mathcal{N}} {p_{n, n'}}.
% \end{align}

%%%%%
\textit{2) \textbf{Loss incurred by coopetition}:} The coopetition loss of an organization $n$ is measured by the revenue of its competing organizations $n' \in \mathcal{N}$, gained from its contribution and competitive intensity, which is calculated as follows:
\begin{align}
\label{eq:gain_competitior}
r_{n'} = \phi_{n'} \gamma_{n, n'} \left[\esilondngen - \esilondgenother\right],
\end{align}
where $\phi_{n'}$ is the organization $n'$'s revenue per unit of contribution gap. We impose that $\xi \leq \phi_{n'}$ to maintain cooperation stability between organizations. This constraint ensures that the compensation rate for contribution disparity does not exceed the marginal benefit that an organization $n'$ derives from another organization's efforts. For example, in the mergers and acquisitions sector, in Facebook’s acquisition of WhatsApp, \$3B in restricted stock units ($\xi$) was structured to vest over four years, effectively capping payouts in proportion to the projected synergy revenue of $\phi_{n'}$, expected from WhatsApp’s integration. This ensures that the compensation for continued cooperation did not exceed the marginal utility Facebook anticipated gaining, aligning with our CFL assumption $\xi \leq \phi_{n'}$ \cite{campbell2014facebook}. In addition, regulatory requirements in sectors such as finance or healthcare may mandate cooperation despite misaligned compensation. The coopetition loss of organization $n$ is given as follows:
% $R_{n}(d^{\text{gen}}_{n}, \boldsymbol{d}^{\text{gen}}_{-n})$, which is calculated by
\begin{align}
\label{eq:loss}
R_{n}(d^{\text{gen}}_{n}, \boldsymbol{d}^{\text{gen}}_{-n}) = \sum\nolimits_{n' \in \mathcal{N}} r_{n'}.
\end{align}

%%%%%%%%%%%%%%%%%%%%%%%%%%%%%%%%%%%%
\subsection{Problem Formulation}
%
%%%%%%%%%%%%%%%%%
% \subsubsection{Client Utility} 
\textit{Client utility} in one training round, when the client chooses a strategy space $d^{\text{gen}}_{n}$, is represented as follows:
\begin{align}
\label{eq:client_utility_one_round}
{U}_{n}(d^{\text{gen}}_{n}, \boldsymbol{d}^{\text{gen}}_{-n}) & = r_{n}(d^{\text{gen}}_{n}, \boldsymbol{d}^{\text{gen}}_{-n}) + P_{n}(d^{\text{gen}}_{n}, {\boldsymbol{d}^{\text{gen}}_{-n}})  \nonumber -  C_{n} - C_{0} \\ & - R_n (d^{\text{gen}}_{n}, \boldsymbol{d}^{\text{gen}}_{-n}),
% & = \psi_n \left[\epsilon_0 - \esilontotaldgen\right]  + \sum_{n' \in \mathcal{N}} {\xi \gamma_{n, {n'}} \left[\esilondngen - \esilondgenother \right]} \nonumber\\ 
% & - \sum_{n' \in \mathcal{N}} \phi_{n'} \gamma_{n, n'} \left[\esilondngen - \esilondgenother \right] \nonumber \\ &-  C_n^{\text{cmp}} \kappa_n (\eta_n  d_{n}^{{\text{mix}}} + \mu_n d_{n}^{{\text{gen}}}) f_n^2 - m \nonumber \\
% & = \psi_{n} \left[\epsilon_0 - \esilontotaldgen \right]  \nonumber\\  &+ \sum_{n' \in \mathcal{N}} \gamma_{n, {n'}} (\xi -\phi_{n'}) \left[\esilondngen - \esilondgenother \right]  \nonumber\\ \
% &  -  \kappa_n C_n^{\text{cmp}}\left[(\eta_n  +\mu_n)\dgen + \eta_n \dloc) \right]f_n2 - m.
\end{align}
% \subsubsection{Social Welfare} 
\textit{Social welfare} is defined as the sum of organizational utilities (i.e., $\sum_{n \in \mathcal{N}} {U}_n(d_n, \boldsymbol{d}^{\text{gen}}_{-n})$).
{\proposed} tackles inter-organizational competition and heterogeneity, aiming to improve collective social welfare. It ensures that maximizing individual client utility aligns with maximizing social welfare through payoff redistribution under the following constraints.
\begin{definition}[Individual Rationality]
\label{def:ir}
Each organization $n$ only chooses to join the CFL Coopetition if its utility is non-negative, specifically, 
% ${U}_{n}(d^{\text{gen}}_{n}, \boldsymbol{d}^{\text{gen}}_{-n}) \ge 0, \forall{n \in \mathcal{N}}.$
\begin{align}
\label{eq:ir}
{U}_{n}(d^{\text{gen}}_{n}, \boldsymbol{d}^{\text{gen}}_{-n}) \ge 0, \forall{n \in \mathcal{N}}.
\end{align}
\end{definition}
\begin{definition}[Budget Balance]
\label{def:bb}
The system will be sustainable without external investment or the server will neither lose nor make a profit \cite{jackson2014mechanism}. In other words, the summation of monetary transfer is zero, i.e., 
% $\sum_{n \in \mathcal{N}} P_n(d^{\text{gen}}_{n}, \boldsymbol{d}^{\text{gen}}_{-n}) = 0.$
\begin{align}
\label{eq:bb}
\sum_{n \in \mathcal{N}}\nolimits P_n(d^{\text{gen}}_{n}, \boldsymbol{d}^{\text{gen}}_{-n}) = 0.
\end{align}
\end{definition}
\section{Potential Game-based proposed mechanism}
\label{sec:proposed_mechanism}
In this section, we demonstrate that the interaction between coopetitive organizations is a weighted potential game. We then use it to derive GenAI-based data generation strategies.

\subsection{Game Formulation}
\textbf{\textit{Game $\mathcal{G}$:}} (Stage Game in the training round $t$)
\begin{itemize}[leftmargin=.3in]
\item \textit{Players:} competitive, heterogeneous, self-interested, rational, and strategic organizations $n \in \mathcal{N}$.
\item \textit{Strategies:} each organization $n$ decides how much data $d^{{\text{gen}}}_{n}$ it should generate in the training round $t$.
\item \textit{Objectives:} each organization $n$ aims to maximize its utility $U_n$ expressed in \eqref{eq:client_utility_one_round}.
\end{itemize}

% \noindent Now, we define the Nash Equilibrium (NE) of the $\mathcal{G}$.
\begin{definition}[Nash Equilibrium of $\mathcal{G}$]
\label{def:nash1}
A strategy profile $\dgenopt$ is a pure strategy Nash Equilibrium (NE) point of the $\mathcal{G}$ if and only if no organization can improve its utility by deviating unilaterally \cite{nash1951non}, i.e., $U_n({\dgenopt}, {\boldsymbol{d}^{*,\text{gen}}_{\boldsymbol{-n}}}) \ge U_n({d^{{\text{gen}}}_{n}}, {\boldsymbol{d}^{*,\text{gen}}_{\boldsymbol{-n}}}).$
% \begin{align}
% \label{eq:nash1}
% U_n({\dgenopt}, {\boldsymbol{d}^{*,\text{gen}}_{\boldsymbol{-n}}}) \ge U_n({d^{{\text{gen}}}_{n}}, {\boldsymbol{d}^{*,\text{gen}}_{\boldsymbol{-n}}}).
% \end{align}
\end{definition}

% At the NE in $\mathcal{G}$, no organization can unilaterally change its strategy to get a higher utility in the training round $t$. 
% \vspace{0.05in}
\topskip=0.05in
\begin{algorithm}[t]
\small
\caption{GenAI-based data generation strategies using FPI}
\label{alg:fpi_method}
\begin{algorithmic}[1]
\Require $\dloc$, $\alpha$, $\beta$, $\delta$, $\varrho$, $\kappa_n$, $C^{\text{cmp}}_n$, $f_n$, $\eta_n$, $\mu_n$,  $z_n$,  $\dgenmin$, $\dgenmax$, tolerance $\epsilon_{\text{tol}}$, maximum iterations $K_{\text{max}}$, A1, A2, A3

\Ensure Data generation strategies $\boldsymbol{d}^{*,\text{gen}} = \{\dgenopt: \forall n \in \mathcal{N}\}$

%%% 
\State Initialize $k = 0$, $\boldsymbol{d}^{\text{gen}}$, Calculate $F^{0}(\boldsymbol{d}^{\text{gen}})$, and Set $F^{-1}(\boldsymbol{d}^{\text{gen}}) = F^{0}(\boldsymbol{d}^{\text{gen}})$

\While{$|F^{k}(\boldsymbol{d}^{\text{gen}}) - F^{k-1}(\boldsymbol{d}^{\text{gen}})| > \epsilon_{\text{tol}}$ and $k \leq K_{\text{max}}$} 
\State $k \gets k + 1$
\For {each $n = 1$ to $N$}

\If {$A_3 \frac{\alpha \beta}{N \varrho} \text{exp}\left(\frac{A_1 - 1}{\varrho}\right) > -A_2$}

\State $\dgen \gets \dgenmin$
\ElsIf{$A_3 \frac{\alpha \beta}{N \varrho} \text{exp}\left(\frac{A_1 - 1}{\varrho}\right) < -A_2$}
\State $\dgen \gets \dgenmax$
\Else
\State $\dgen \gets \biggl[-\frac{A_2 N \varrho}{\alpha \beta} \text{exp}\left(-\frac{A_1 - 1}{\varrho} \right)\biggr] ^ {\frac{-1}{\beta+1}} - \dloc$
\State Clip $\dgen$ to the range $[\dgenmin, \dgenmax]$
\EndIf
\EndFor
\State Update $\boldsymbol{d}^{\text{gen}}$ and Calculate $F^{k}(\boldsymbol{d}^{\text{gen}})$
\EndWhile

\State \Return $\boldsymbol{d}^{*,\text{gen}}$
\end{algorithmic}
\end{algorithm}
\setlength{\textfloatsep}{2pt}% Remove \textfloatsep vertical space after the algorithm
%%%%%%%%%%%%%%%%%%%%%%%%%%%%%%%%%%%%%%
\subsection{Nash Equilibrium Analysis}
Directly finding the NE of $\mathcal{G}$ is challenging because deriving the closed-form expression of the fixed point of organizations' best response mapping is complicated. To analyze $\mathcal{G}$, we show that $\mathcal{G}$ is a weighted potential game \cite{la2016potential}. Then, we derive the NE of $\mathcal{G}$ by solving the minimization problem over the corresponding weighted ordinary potential function.

\begin{theorem}
\label{def:weighted_pg}
$\mathcal{G}$ is a weighted potential game with the potential function $F(\boldsymbol{d}^{\text{gen}})$, which is calculated by
\begin{align}
\label{eq:theorem1}
F(\boldsymbol{d}^{\text{gen}}) = \esilontotaldgen - \sum_{n \in \mathcal{N}} \frac{\kappa_nC_n^{\text{cmp}} (\eta_n  +\mu_n)\dgen f_n^2}{z_n},
\end{align}
\end{theorem}
\noindent where ${z_n} = \sum_{n' \in \mathcal{N}} \left[\gamma_{n, {n'}} (\xi -\phi_{n'})\right] - \psi_n < 0.$
\begin{proof}
    % Please refer to Appendix \ref{proof:theorem1} for more details.
    The proof can be referred to \cite{la2016potential}.
\end{proof}
Since $\mathcal{G}$ is a weighted potential game and $z_n < 0$, its NE corresponds to the (global or local) optimal solution to the following minimization problem%, i.e., 
\begin{subequations}
\label{eq:optim_func}
\begin{align}
\label{eq:opt_func}
\min_{\boldsymbol{d}^{\text{gen}}} \quad & \potfunc \\  
\noindent \textrm{s.t.} \quad  & \dgenmin \leq \dgen \leq \dgenmax, \dgen\in \mathbb{Z}_+, \\
&  \eqref{eq:ir}, \eqref{eq:bb}.
\end{align}
\end{subequations}
\eqref{eq:optim_func} is a non-convex optimization problem. To overcome this non-convexity challenge, we relax the value of $d^{{\text{gen}}}_n$ (in number of data samples) to the interval $[\dgenmin, \dgenmax]$. The problem \eqref{eq:optim_func} can therefore be recast as follows:
\begin{subequations}
\label{eq:optim_func2}
\begin{align}
\label{eq:optim_fun3}
\min_{\boldsymbol{d}^{\text{gen}}} \quad & \potfunc \\  
\noindent \textrm{s.t.} \quad  & \dgenmin \leq \dgen \leq \dgenmax,\dgen\in \mathbb{R}_+, \\
&  \eqref{eq:ir}, \eqref{eq:bb}.
\end{align}
\end{subequations}
By solving \eqref{eq:optim_func2}, we show NE of $\mathcal{G}$.

\begin{theorem}
    $\mathcal{G}$ possesses a NE of $\boldsymbol{d}^{*,\text{gen}} = \{\dgenopt: \forall n \in \mathcal{N}\}$ under different cases, where we define the following variables
\begin{subequations}
\label{eq:kkt_conditions}
\begin{empheq}[left={}\empheqlbrace]{align}
& A_1 = \frac{1}{N} \sum\nolimits_{n \in \mathcal{N}}\left[\alpha(\dloc + \dgenopt)^{-\beta} - \delta \right] \\ 
& A_2 =  {\kappa_n C_n^{\text{cmp}} (\eta_n  +\mu_n) f_n^2}/{z_n} \\ 
& A_3 = (\dloc + \dgenopt)^{-\beta - 1}
\end{empheq}
\end{subequations}
\begin{itemize}[leftmargin=*]
    \item \underline{Case 1}: $\dgenopt = \dgenmin$ iff 
$
    A_3 \frac{\alpha \beta}{N \varrho} \text{exp}\left(\frac{A_1 - 1}{\varrho}\right) > -A_2.
$

    \item \underline{Case 2}: $\dgenopt = \dgenmax$ iff 
$
    A_3 \frac{\alpha \beta}{N \varrho} \text{exp}\left(\frac{A_1 - 1}{\varrho}\right) < -A_2.
$

    \item \underline{Case 3}: $\dgenmin <\dgenopt < \dgenmax$ and $\dgenopt = \biggl[-\frac{A_2 N \varrho}{\alpha \beta} \text{exp}\left(-\frac{A_1 - 1}{\varrho} \right)\biggr] ^ {\frac{-1}{\beta+1}} - \dloc$. 
\end{itemize}
\end{theorem}
\begin{proof}
We first prove that the optimization problem~\eqref{eq:optim_func2} is a strict convex function by verifying that the Hessian matrix of $\potfunc$ is positive definite (i.e., $\bigtriangledown^2\mathcal{H}(\potfunc) \succ 0$). The optimization problem~\eqref{eq:optim_func2} also satisfies Slater's condition. Therefore, Karush–Kuhn–Tucker (KKT) conditions are necessary and sufficient to solve it. Finally, we utilize the fixed-point iteration (FPI) method \cite{bauschke2011fixed} to approximately find the solutions. We omit the details here due to space limitations. 
% This completes the proof.
\end{proof}

The procedure for computing approximate solutions $\boldsymbol{d}^{*,\text{gen}}$ is presented in Algorithm~\ref{alg:fpi_method}. 
The solution is verified against conditions in \eqref{eq:ir} and \eqref{eq:bb}.
The algorithm has a time complexity of $\mathcal{O}(NK)$, where $K$ is the number of iterations required. 
\section{Performance evaluation}
\label{sec:performance_evaluation}
\subsection{Experiment Setup}
We consider a CFL setup of $N = 10$ organizations (typical in CFL~\cite {kairouz2021advances}). The average competitive intensity $\bar\gamma = \{0.246, 0.492, 0.746\}$ is computed as the mean over all organizations with competitive intensities drawn from $U(0, 0.5), U(0, 1), \text{ and } U(0.5, 1)$, respectively. To model statistical heterogeneity among organizations, we sample label proportions $p\sim\textit{Dir}{(\alpha_D)}$, where $\alpha_D = \{0.1, 0.5, 0.9\}$ is the Dirichlet distribution parameter. The neural network architecture includes two convolutional layers followed by ReLU and max pooling, and a dense layer. The batch size is 32 and the learning rate is 0.01. Parameter settings are shown in Table~\ref{tab:parameters}. 
\begin{table}[t!]
    \caption{PARAMETER SETTINGS}
    \small
    \label{tab:parameters}
    \centering
    \begin{tabular}{ll|ll}
        \hline
    \rowcolor[HTML]{EFEFEF} 
      \textbf{Param}  & \textbf{Value} & \textbf{Param} & \textbf{Value}  \\\hline
    $\kappa_n$ & $U(2\times10^{-18}, 5\times10^{-18})$  &  N & 10 \\
    
      $\dloc$ & $U(1000, 3000)$ samples & $\varrho$ & 20   \\
       Elec. price & 0.3438 euro per kWh\tablefootnote{https://selectra.ie/energy/guides/electricity-prices-ireland}   &$d_{\text{min}}^{\text{gen}}$ & 0 samples\\ 
        $\phi_n$ & $U(2\times10^{2},3\times10^{2})$  & $\gamma_{n, n'}$ & $U(0,1)$ \\ 

      $\xi_n$  & $20$ & $f_n$ & $U(1, 2)$  GHz  \\ 
    $\psi_n$ & $U(6\times10^{2},9\times10^{2})$ & $d_{\text{max}}^{\text{gen}}$ & 3000 samples\\
         \hline
    \end{tabular}
\end{table}

\textbf{Baselines.} Existing work offers limited applicable baselines for competitive-aware data generation among organizations in heterogeneous and coopetitive CFL settings. We therefore adopt the following approaches as comparative baselines.
\begin{itemize}
    % \item \textbf{Without Payoff Redistribution (WPR).} In this approach, organizations will only benefit from the global model performance \cite{huang2023duopoly, huang2024coopetition} (i.e., $P_{n}(d^{\text{gen}}_{n}, {\boldsymbol{d}^{\text{gen}}_{\boldsymbol{-n}}}) = 0$).
    \item \textbf{Vanilla CFL (VCFL).} Traditional CFL without GenAI-based data augmentation approaches (i.e., $\dgen = 0$).
    \item  \textbf{Without Competition among Organizations (WCO).} In this WCO scheme, there is no competition among organizations (i.e., $\gamma_{n,n'} = 0, \forall n, n' \in \mathcal{N}, n' \neq n$).
    \item \textbf{Random Data Generation (RaDG).} In this RaDG scheme, organizations randomly generate the amount of GenAI-based augmented data for local training.

\end{itemize}
\subsection{Simulation Results and Analysis}
\underline{\textbf{Scaling laws for CFL.}} As shown in Fig.~\ref{fig:hyperparameter}, each  $\alpha_D$ yields a corresponding set of hyperparameters $\{\alpha, \beta, \gamma\}$ characterizing the learning curve of local models trained on Fashion-MNIST dataset\footnote{Our mechanism and weighted-potential-game analysis are dataset-agnostic.}. As the amount of GenAI-generated data increases, the local learning error decreases significantly across different $\alpha_D$ values. This reduction saturates once the generated data reaches a sufficient volume, at which point the loss stabilizes. 
% Notably, our hyperparameterized scaling law yields learning curves that closely match the empirical progression, validating its effectiveness. These results demonstrate that GenAI-based data augmentation not only improves local model performance but also enables \textit{predictable} and \textit{tunable} improvements in downstream tasks in federated settings. Consequently, our learning curves provide a principled basis for guiding decisions on whether and how to scale data generation under different heterogeneous settings.
\begin{figure}[t!]
	\centering
	\includegraphics[width=0.65\linewidth]{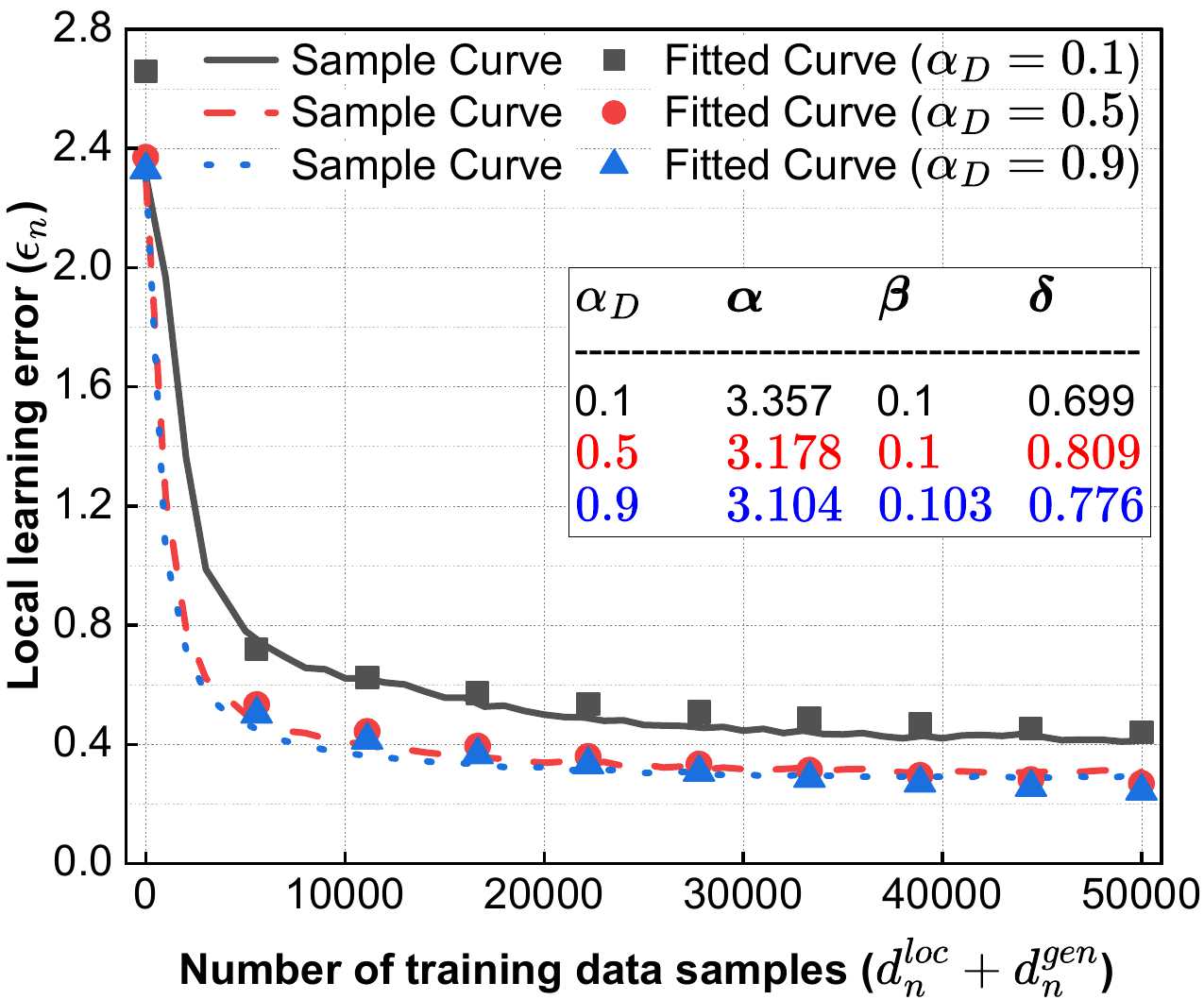}
	\caption{$\epsilon_n$ with respect to the number of local training data.}
    \vspace{3pt}
    \label{fig:hyperparameter}
\end{figure} 
\begin{figure}[t!]
	\centering
	\includegraphics[width=0.96\linewidth]{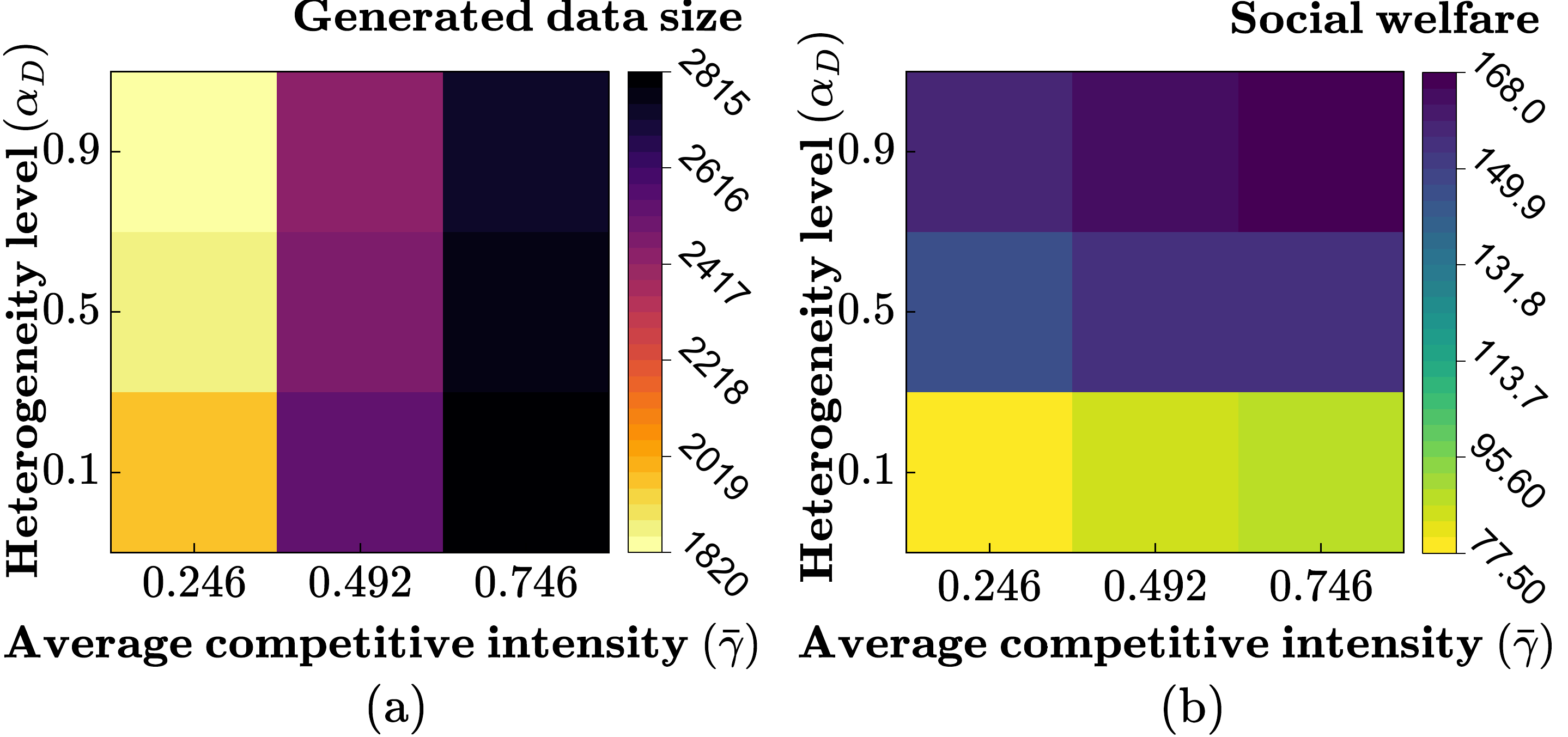}
	\caption{Impact of $\bar\gamma$ and $\alpha_{D}$ on the organization's data generation strategy and the social welfare of \proposed.}
    \vspace{5pt}
    \label{fig:feasibility}
\end{figure} 
\underline{\textbf{Impacts of $\bar\gamma$ and $\alpha_D$ on \proposed.}} In Fig.~\ref{fig:feasibility}, we analyze the joint impact of statistical heterogeneity ($\alpha_D$) and competitive intensity ($\bar\gamma$) on the required volume of generated data and social welfare.  Under high heterogeneity settings (e.g., $\alpha_D = 0.1$) with increasing $\bar\gamma$, the severe data class imbalance across competitive organizations necessitates substantial data augmentation to maintain convergence towards the optimal global model performance.  This increases computational overhead for organizations, thereby reducing overall social welfare. Furthermore, under strong competitive market conditions (e.g., $\bar\gamma = 0.746$), organizations may experience high losses from cooperation, such as exposing proprietary data features that contribute to their business advantage. As a result, it can further diminish the incentives for cooperation and also lower overall social welfare. Therefore, insights for \textbf{\#RQ1} are provided.

\underline{\textbf{Efficiency of {\proposed}.}} Fig.~\ref{fig:social_welfare} demonstrates the effectiveness of {\proposed} compared to other baseline methods under varying levels of  $\bar\gamma$ and $\alpha_D$. Under a given heterogeneity level ($\alpha_D = 0.5$), Fig.~\ref{fig:social_welfare}(a) shows that {\proposed} consistently achieves the highest social welfare as competition level intensifies. This advantage arises from its ability to guide organizations in effectively determining the amount of GenAI-based data for local learning based on local resources. Furthermore, the embedded payoff redistribution mechanism compensates organizations with higher contributions (e.g., those generating more data), thereby mitigating coopetitive losses and reinforcing long-term participation incentives. While RaDG exhibits a similar upward trend, its randomized data generation approach limits efficient utilization of local capacity, resulting in lower welfare than {\proposed}. In contrast, VCFL  yields the lowest social welfare, which slightly declines as $\bar\gamma$ increases. This is attributed to organizations’ reluctance to generate more data, driven by concerns over losing market advantage in a competitive environment.  WCO is self-explanatory because it does not consider the competition factors and organizations with higher contributions receive only proportional returns from the global model, thereby maintaining stable social welfare.

% \vspace*{0.12in}  
\begin{figure}[t!]
	\centering
	\includegraphics[width=\linewidth]{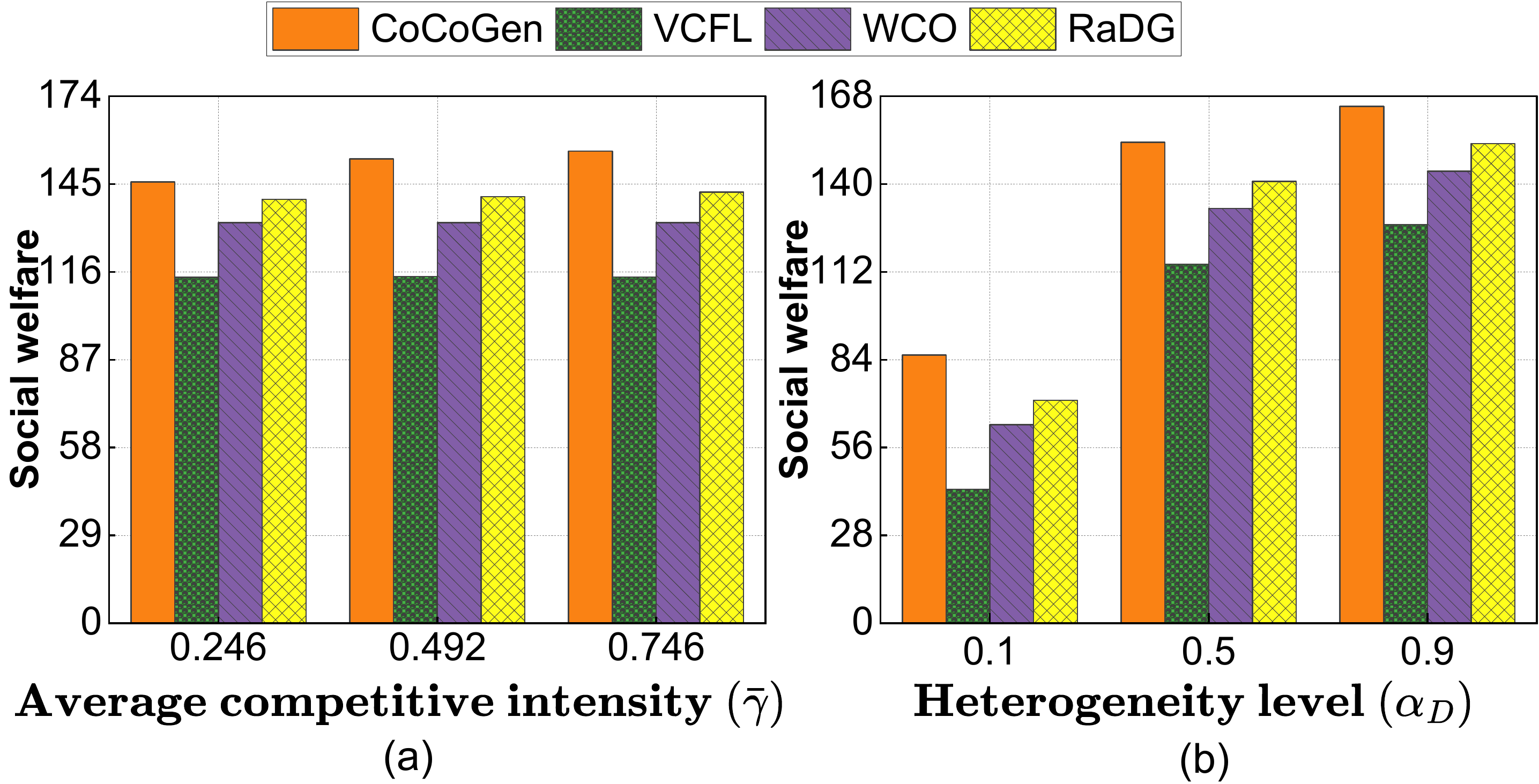}
	\caption{Efficiency of {\proposed} compared to baseline methods.}
    \vspace{5pt}
    \label{fig:social_welfare}
\end{figure} 

Under a fixed competitive intensity ($\bar\gamma = 0.492$), as illustrated in  Fig.~\ref{fig:social_welfare}(b), {\proposed} and other methods exhibit increasing social welfare as $\alpha_D$ increases. Notably, {\proposed} continues to outperform all baselines across $\alpha_D$ values,\textit{ highlighting the effectiveness of its data generation and payoff redistribution mechanisms}. As data distributions become more heterogeneous (i.e., smaller $\alpha_D$), the social welfare gap between {\proposed} and the baselines widens. Specifically, the relative improvements in social welfare achieved by {\proposed} over VCFL, WCO, and RaDG grow from around 22.87\%, 12.58\%, and 7.21\% at $\alpha_D = 0.9$ to 50\%, 25.93\%, and 16.89\% at $\alpha_D = 0.1$, respectively. These results underscore the benefits of proposed data generation mechanisms in heterogeneous and competitive CFL environments, thereby addressing \textbf{\#RQ2}.

% \begin{figure}[t!]
% 	\centering
% 	\includegraphics[width=0.8\linewidth]{Figures/sec_performance_evaluation/social_welfare_alphaD.pdf}
% 	\caption{Test}
%     \vspace{-10pt}
%     \label{fig:social_welfare_withvarying_alphaD}
% \end{figure} 

% \begin{figure}[t!]
% 	\centering
% 	\includegraphics[width=0.8\linewidth]{Figures/sec_performance_evaluation/social_welfare_gamma.pdf}
% 	\caption{Test}
%     \vspace{-10pt}
%     \label{fig:social_welfare_withvarying_gamma}
% \end{figure} 
\section{Conclusion}

In this paper, we have studied the coexistence of coopetition and statistical heterogeneity in CFL, aiming to understand their joint impact on organizational behavior and system-wide social welfare. To this end, we have proposed {\proposed}, a coopetitive-compatible data generation framework, that models competition and statistical heterogeneity through learning loss functions and utility-based functions. We have proved that each training round constitutes a weighted potential game, enabling the derivation of GenAI-based data generation strategies that maximize social welfare. Experimental results have shown that higher competitive intensity (i.e., $\bar\gamma \gg$) and statistical heterogeneity (i.e., $\alpha_D \ll$) significantly raise GenAI-based data generation demands while reducing social utility and vice versa. Furthermore, {\proposed} has achieved superior social welfare in comparison to a couple of baseline methods, showing the effectiveness of {\proposed}.

% \appendix
% \input{Appendix/Theorem1}
% \section*{Acknowledgment}
% This publication has emanated from research conducted with the financial support of Science Foundation Ireland under Grant number 18/CRT/6222. 

% For the purpose of Open Access, the author
% has applied a CC BY public copyright licence to any Author Accepted Manuscript version arising
% from this submission.
% \clearpage

\section*{Acknowledgments}
{\small{This publication has been conducted with financial support of Taighde Éireann – Research Ireland under Grant number 18/CRT/6222 and the School of Computer Science and Statistics, TCD. The work of Quoc-Viet Pham is supported in part by Research Ireland under the European Innovation Council CHIST-ERA SHIELD project (Project No. 216449, Award No. 19226). A CC BY licence applies to any Author Accepted Manuscript from this submission.}}

\bibliographystyle{IEEEtran}
% \bibliography{Refs}

\end{document}